\documentclass{article}

\usepackage{titling}

\pretitle{\begin{flushleft}\LARGE\bfseries}
\posttitle{\par\end{flushleft}}
\preauthor{\begin{flushleft}\large}
\postauthor{\par\end{flushleft}}
\predate{\begin{flushleft}}
\postdate{\par\end{flushleft}\vskip 0.5em}
\usepackage[svgnames]{xcolor}

\usepackage[customcolors,shade]{hf-tikz}
\usepackage{empheq,multicol,titlesec}
\usepackage[most]{tcolorbox}
\newtcbox{\mymath}[1][]{%
    nobeforeafter, math upper, tcbox raise base,
    enhanced, colframe=blue!30!black,
    colback=gray!30, boxrule=1pt,
    #1}
\usepackage[utf8]{inputenc} % allow utf-8 input
\usepackage[T1]{fontenc}    % use 8-bit T1 fonts
\usepackage{hyperref,amsmath,amsthm,amssymb,comment,tabularx}       % hyperlinks
\usepackage{url}            % simple URL typesetting
\usepackage{booktabs}       % professional-quality tables
\usepackage{amsfonts}       % blackboard math symbols
\usepackage{nicefrac}       % compact symbols for 1/2, etc.
\usepackage{microtype}      % microtypography
\usepackage[linesnumbered,ruled]{algorithm2e}
\usepackage{graphicx,caption,paralist,color,wrapfig}

\newtheorem{theorem}{Theorem}
\newcommand\numberthis{\addtocounter{equation}{1}\tag{\theequation}}

\usepackage{multirow}
\usepackage{colortbl, array}
\usepackage{hhline}

\definecolor{rulecolor}{RGB}{0,71,171}
\definecolor{tableheadcolor}{RGB}{204,229,255}

\title{A mixture model for aggregation of multiple pre-trained weak
  classifiers}

\author{Rudrasis Chakraborty\textsuperscript{*}, Chun-Hao Yang\textsuperscript{*} and Baba C. Vemuri\\
Department of CISE, University of Florida, FL 32611, USA\\
\textsuperscript{* Equal Contribution} \\
{\tt\small \{rudrasischa, baba.vemuri\}@gmail.com ~~~~ \{chunhaoyang\}@ufl.edu}
}
\date{}

\begin{document}

\maketitle
\begin{abstract}
Deep networks have gained immense popularity in Computer Vision and
other fields in the past few years due to their remarkable performance
on recognition/classification tasks surpassing the state-of-the
art. One of the keys to their success lies in the richness of the
automatically learned features.  In order to get very good accuracy,
one popular option is to increase the depth of the network. Training
such a deep network is however infeasible or impractical with moderate
computational resources and budget.  The other alternative to increase
the performance is to learn multiple weak classifiers and boost their
performance using a boosting algorithm or a variant thereof. But, one
of the problems with boosting algorithms is that they require a
re-training of the networks based on the misclassified
samples. Motivated by these problems, in this work we propose an
aggregation technique which combines the output of multiple weak
classifiers. We formulate the aggregation problem using a mixture
model fitted to the trained classifier outputs. Our model does not
require any re-training of the ``weak'' networks and is
computationally very fast (takes $<30$ seconds to run in our
experiments). Thus, using a less expensive training stage and without
doing any re-training of networks, we experimentally demonstrate that
it is possible to boost the performance by $12\%$. Furthermore, we
present experiments using hand-crafted features and improved the
classification performance using the proposed aggregation
technique. One of the major advantages of our framework is that our
framework allows one to combine features that are very likely to be of
distinct dimensions since they are extracted using different
networks/algorithms. Our experimental results demonstrate a
significant performance gain from the use of our aggregation technique
at a very small computational cost.
\end{abstract}

\section{Introduction}\label{intro}

Deep convolution neural networks (CNNs) have recently gained immense
attention in computer vision and machine learning communities mainly
because it's superior performance in various applications including
image classification \cite{he2016deep,krizhevsky2012imagenet,lecun1995convolutional}, object detection \cite{girshick2014rich,iandola2014densenet,ren2015faster}, face
detection/recognition \cite{sun2014deep,lecun2015deep,parkhi2015deep,sun2014deepb} and many others. These networks usually
consist of a stack of convolution layers and fully connected layers
with pooling and non-linearity in between. By stacking multiple
layers, deep network can essentially extract complex features which
are more discriminative than features extracted by traditional machine
learning algorithms \cite{sermanet2013overfeat,simonyan2014very,szegedy2015going,zeiler2014visualizing}. Krizhevsky et al. \cite{krizhevsky2012imagenet} proposed a deep
CNN architecture (dubbed AlexNet) which performed exceptionally
well on ImageNet image classification dataset. The tremendous success
of AlexNet lead to a flurry of research activity in the community
resulting in a variety of deep CNN architectures for face recognition,
action recognition etc. etc.

As there are no specific guidelines regarding the choice of the depth
and width of the network, a significant amount of research has focused
on finding heuristics to determine these parameters to obtain the
``optimal'' network for the target application. This resulted in very
deep networks like DenseNet201 (of depth 201)
\cite{iandola2014densenet}, ResNet50 (of depth 168) \cite{he2016deep},
InceptionResnetv2 (of depth 572) \cite{szegedy2017inception}, Xception
(of depth 126) \cite{chollet2016xception} and others. Though these
very deep networks perform well on large datasets like ImageNet
\cite{deng2009imagenet}, JFT dataset \cite{hinton2015distilling} and
others, retraining these networks for small datasets or different
target applications is difficult due to their enormous size (in terms
of number of parameters). This raises the question, is it possible to
combine multiple ``weak'' networks (of smaller depth and hence lower
accuracy) and boost the performance significantly over each individual
network in the combination?

In response to the above question, recently, several researchers
proposed algorithms that construct a combination of different networks
to achieve improved performance. The basic idea of these methods have
been borrowed from traditional ML algorithms like bagging
\cite{dietterich2000ensemble} and boosting
\cite{schapire2003boosting}. Some of these methods rely on a weighted
combination of different networks
\cite{schwenk1997adaboosting,schwenk2000boosting,saberian2011multiclass}. While
boosting methods like the Diabolo classifier \cite{schwenk1998diabolo}
and the multi-column deep network
\cite{agostinelli2013adaptive,ciregan2012multi} focus on retraining
networks based on the previously misclassified samples. In
multi-column CNN, the authors train multiple CNNs simultaneously so
that a linear combination of these CNNs boost the performance and
serve as the final predictor. Recently, the authors in
\cite{moghimi2016boosted} proposed a boosting technique named BoostCNN
where similar to Adaboost
\cite{schwenk1997adaboosting,schwenk2000boosting}, they learn CNNs
sequentially on the mistakes from the earlier networks in the
sequence. Essentially, they built a deep CNN where the final network
output is aligned with the boosting weights. Though this sequential
approach is less expensive than multi-column deep network, this still
needs training of the CNNs which is time consuming. Very recently,
several significantly deep networks have been proposed in literature
\cite{szegedy2017inception,he2016deep}. Though these networks perform
very well, training takes a significant amount of time and hence
retraining is not computationally feasible. Even using transfer
learning \cite{pan2010survey}, sometimes it is not computationally
viable to train/fine tune these networks.
 
In this work, we propose a novel framework which takes multiple
pre-trained ``weak'' CNNs as input and outputs a probabilistic model
which is an aggregation of the pre-trained CNNs. We formulate the
problem of combining weak CNNs as a mixture model of the distributions
learned from the output of the deep networks. Our formulation can also
deal with features of different dimensions and provide a boosted
performance. Hence, we have two sets of experiments one to show the
performance boost on multiple weak deep networks and the other
experiment to show performance boost on multiple popular hand crafted
features. In practice, our method takes $<30$ seconds of additional
time to achieve the boosted performance. One of the key advantages of
our proposed framework is unlike previous boosting techniques, {\it it
  does not require any re-training of CNNs}. We show that our model
requires a simple optimization on a hypersphere which is solved using
a Riemannian gradient descent based approach. We have incorporated
both the parametric and non-parametric models for representing the
combination of networks and have shown that both these models achieve
boosted performance of the aggregation technique when compared to each
of the weak network classifiers. Through experiments, we show that on
CIFAR-10 data \cite{krizhevsky2014cifar}, using $20$ weak classifiers
of depth $<20$, our parametric model improved the accuracy by about
$8\% \sim 12\%$. On MNIST data \cite{lecun1998mnist,deng2012mnist},
using $20$ weak classifiers of depth $2$, our model achieves $2\% \sim
3\%$ improvement in classification accuracy.

Rest of the paper is organized as follows. In section \ref{theory}, we
present the framework for combination of ``weak'' networks. Section
\ref{results} contains various experiments conducted to depict the
performance of the proposed technique for improved performance. In
section \ref{conc} we draw conclusions.

\section{An aggregation of multiple weak networks}
\label{theory}

In this section, we propose both parametric and non-parametric models
to combine multiple ``weak'' networks in order to boost the overall
performance. In any deep network used for classification, the output
is a probability vector corresponding to the probability of the given
test data belonging to set of classes under consideration. In this
paper, we propose to exploit the geometry of the space of probability
densities. However, this space is a statistical manifold and the
natural metric on it is the well known Fisher-Rao metric
\cite{amari2016information}, which is difficult to compute. Hence, a
square root parameterization of the density is used to map the density
on to a unit Hilbert sphere whose geometry is fully known. Further, the
natural metric on the sphere can be used in all computations as it is
in closed form and is computationally efficient. We now present the
relevant basic concepts of differential geometry as applied to the
sphere that are needed in this work.

\subsection{Review of Basic Riemannian Geometry of $\mathbf{S}^N$}

The N-dimensional sphere, $\mathbf{S}^N$, is a Riemannian manifold
with constant positive curvature and is the simplest and widely
encountered manifold in many application domains.  In following
paragraph, we will present a very brief review of the relevant
differential geometry concepts of $\mathbf{S}^N$.
\vspace{0.2cm} 

{\bf Geodesic distance:} We will use the arc length distance as the
geodesic distance on $\mathbf{S}^N$. The arc length distance,
$d_{\text{arc}}: \mathbf{S}^N \times \mathbf{S}^N \rightarrow
\mathbf{R}$ is defined as follows:
\begin{align*}
d_{\text{arc}}(\mathbf{x},\mathbf{y}) = \cos^{-1}\left( \mathbf{x}^t\mathbf{y}\right),
\end{align*}
where $\mathbf{x}, \mathbf{y} \in \mathbf{S}^N$.
\vspace{0.2cm}

{\bf Exponential map:} Let, $\mathbf{x} \in \mathbf{S}^N$. Let
$\mathcal{B}_r\left(\mathbf{0}\right) \subset
T_{\mathbf{x}}\mathbf{S}^N$ be an open ball centered at the origin in
the tangent space at $\mathbf{x}$, where $r$ is the {\it injectivity
  radius} of $\mathbf{S}^N$ \cite{Berger}. Then, we can define the
\textbf{Exponential map}, $\mathsf{Exp}_{\mathbf{x}} :
T_{\mathbf{x}}\mathbf{S}^N \rightarrow \mathbf{S}^N$ as:
\begin{align*}
\mathsf{Exp}_{\mathbf{x}}\left(\mathbf{v}\right) = \cos\left(\|\mathbf{v}\|\right)\mathbf{x}  + \sin\left(\|\mathbf{v} \|\right)\frac{\mathbf{v}}{\|\mathbf{v}\|}, 
\end{align*}
where, $\mathbf{v} \in T_{\mathbf{x}}\mathbf{S}^N$. The Exponential
map maps a tangent vector $\mathbf{v}$ to a point on the great circle
along the direction $\mathbf{v}$ and with distance $\|\mathbf{v}\|$
from $\mathbf{x}$. Note that on $\mathbf{S}^N$, $r=\pi/2$.
\vspace{0.2cm}

{\bf Inverse Exponential map:} Inside
$\mathcal{B}_r\left(\mathbf{0}\right)$, $\mathsf{Exp}_{\mathbf{x}}$ is
a diffeomorphism, hence, the inverse exists and we can define the
inverse of the Exponential map by $\mathsf{Exp}^{-1}_{\mathbf{x}} :
\mathcal{U} \rightarrow \mathcal{B}_r\left(\mathbf{0}\right)$ and is
given by
\begin{align*}
\mathsf{Exp}^{-1}_{\mathbf{x}}\left(\mathbf{y}\right) = \frac{\theta}{\sin \theta} \left( \mathbf{y} - \mathbf{x} \cos \theta \right),
\end{align*}
where $\mathcal{U} =
\mathsf{Exp}_{\mathbf{x}}\left(\mathcal{B}_r\left(\mathbf{0}\right)\right)
$ and $\theta = d_{\text{arc}}\left(\mathbf{x},\mathbf{y}\right)$.
\vspace{0.2cm}

{\bf Shortest Geodesic curve:} Let $\mathbf{x} \in \mathbf{S}^N$ and
$\mathbf{y} \in
\mathsf{Exp}_{\mathbf{x}}\left(\mathcal{B}_r\left(\mathbf{0}\right)\right)$. Then,
the shortest geodesic curve between $\mathbf{x}$ and $\mathbf{y}$ is a
function $\Gamma_{\mathbf{x}}^{\mathbf{y}}: \mathbf{R} \rightarrow
\mathbf{S}^N$ given by:
\begin{align*}
\Gamma_{\mathbf{x}}^{\mathbf{y}}\left(t\right) = \mathsf{Exp}_{\mathbf{x}}\left(t\mathsf{Exp}^{-1}_{\mathbf{x}}\left(\mathbf{y}\right) \right)
\end{align*}

\subsection{A parametric model for the aggregation of networks}

Let, $N_1, \cdots, N_m$ be the ``weak'' networks that we want to
combine to achieve an improved performance. Let $I \in \mathcal{I}$ be
an input image, where $\mathcal{I}$ is the given set of image
data. Let $f_1, \cdots, f_m \in \mathbf{R}^{c}$ be the output of the
networks, where $f_i$ is the output of $N_i$, i.e., $f_i = N_i(I)$,
and $c$ is the number of classes. Here $f_i$ can be viewed as the
probability vector of size $c$, containing the probabilities of an
image $I$ belonging to each of the $c$ classes. We use the square-root
parametrization to map $f_i$ on to the hypersphere
$\mathbf{S}^{c-1}$. To make the notation more concise, for network
$N_i$, we define a map $\mathcal{F}_i: \mathcal{I}\rightarrow
\mathbf{S}^{c-1}$ as
\begin{align*}
I \mapsto \sqrt{N_i(I)},
\end{align*}
where the square-root is taken element-wise.

Let $\left\{\mathcal{I}_j\right\}_{j=1}^c$ be the partition of the
data $\mathcal{I}$. We assume that for the $i^{th}$ network and for
the $j^{th}$ class, the features $\left\{
\mathcal{F}_{i}\left(I_k\right) | I_k \in \mathcal{I}_j,
k=1,...,|\mathcal{I}_j| \right\}$ are independent and identically
distributed with a Gaussian distribution $p_{ij} =
\mathcal{N}\left(\mu_{ij}, \sigma_{ij}\right)$ on $\mathbf{S}^{c-1}$
with location parameter $\mu_{ij} \in \mathbf{S}^{c-1}$ and scale
parameter, $\sigma_{ij}>0$, i.e., for each $i, j$,
\begin{align} \label{theory:eqn:gau_assumption}
\left\{\mathcal{F}_{i}(I_k) | I_k \in \mathcal{I}_j, k=1,...,|\mathcal{I}_j|\right\} \stackrel{i.i.d}{\sim} \mathcal{N}\left(\mu_{ij}, \sigma_{ij}\right)
\end{align}

On $\mathbf{S}^{c-1}$, we will use the Gaussian distribution,
$\mathcal{N}\left(\mu,\sigma\right)$, as defined in \cite{chakraborty2017statistics}. Let $X$
be an $\mathbf{S}^{c-1}$ valued random variable, then the p.d.f. is
given by:
\begin{align}
f_X(x) = \frac{1}{C(\sigma)}
\exp\left(-\frac{d^2(x,\mu)}{2\sigma^2}\right),
\end{align}
where $d$ is the geodesic distance on
$\mathbf{S}^{c-1}$. $C(\sigma)$ is the normalizing constant.
This distribution, $p_{ij}$, gives the probability of a feature
coming from the $i^{th}$ network and belonging to the $j^{th}$
class.

Let $\left\{\alpha_{i}\right\}_{i=1}^m$ be the weights associated with
the networks such that, they satisfy the affine constraint, i.e.,
\begin{align*}
(\forall i) \alpha_{i} &\geq 0 \\
\sum_{i=1}^m \alpha_{i} &=1
\end{align*}
Now, we will use these weights to define a mixture to model the
combination of these networks. For each class $j$, we define the
probability density, $p_j: \mathcal{I} \to \mathbf{R}$ by $p_j =
\sum_i \alpha_i \left( p_{ij} \circ \mathcal{F}_i \right)$. Hence, for
all $I \in \mathcal{I}$,
\begin{align*}
p_j(I) &= \sum_i \alpha_i p_{ij}(\mathcal{F}_i(I)). 
\end{align*}
Clearly, $p_j(I) \geq 0$ for all $I \in \mathcal{I}$. And because of
the affine constraint on $\left\{\alpha_i\right\}$, $p_j$ is a valid
probability density, for all $j$. Each $p_j$ will represents an
ensemble of the learned models for all the networks. Now, in the
prediction phase, we will assign the test image to the class which
maximizes this probability value.

We define the prediction by our ensemble classifier $p:\mathcal{I}
\to \triangle^{c}$ by $p(I) = \left(\frac{p_1(I)}{\sum_j
    p_j(I)},...,\frac{p_c(I)}{\sum_j p_j(I)} \right)^t$. It is easy to
see that given the image $I$, this is a probability vector since
\begin{align*}
\sum_{i=1}^c \frac{p_i(I)}{\sum_j p_j(I)} = \frac{ \sum_{i=1}^c p_i(I)}{\sum_j p_j(I)} = 1.
\end{align*}
\vspace*{0.2cm} 

{\bf Training the model:} Now we have the training data denoted by,
$\mathcal{I}^{\text{train}} \subset \mathcal{I}$, that is used to
learn the unknown parameters $\{\alpha_i, \mu_{ij},
\sigma_{ij}\}_{i,j}$, and the test data denoted by,
$\mathcal{I}^{\text{test}}\subset \mathcal{I}$. Though, it is possible
for one to learn $\left\{\mu_{ij}, \sigma_{ij}\right\}_{i,j}$,
instead, we use the Fr\'{e}chet mean (FM)\cite{frechet} on
$\left\{\mathcal{F}_i\right\}_{I \in \mathcal{I}^{\text{train}}_j}$ to
get the estimate $\hat{\mu}_{ij}$ and use the sample standard deviation
within $\left\{\mathcal{F}_i\right\}_{I \in
  \mathcal{I}^{\text{train}}_j}$ to get the estimate
$\hat{\sigma}_{ij}$, i.e.,
\begin{align}
\label{theory: parameter}
\hat{\mu}_{ij} &= \arg\min_{\mu \in \mathbf{S}^{c-1}} \frac{1}{\left\vert \mathcal{I}^{\text{train}}_j \right\vert} \sum_{I \in \mathcal{I}^{\text{train}}_j} d_{\text{arc}}^2(\mathcal{F}_i(I),\mu) \\
\hat{\sigma}_{ij} &= \sqrt{\frac{1}{\left\vert \mathcal{I}^{\text{train}}_j \right\vert} \sum_{I \in \mathcal{I}_j} d_{\text{arc}}^2(\mathcal{F}_i(I),\hat{\mu}_{ij})} \\
\hat{C}(\hat{\sigma}_{ij}) &= \left[ \sum_{I \in \mathcal{I}^{\text{train}}} \exp\left(-\frac{d_{\text{arc}}^2(\mathcal{F}_i(I),\hat{\mu}_{ij})}{2\hat{\sigma}_{ij}^2}\right)\right]^{-1}
\end{align} 

In this work, rather than optimizing the minimization problem to get
the FM, we will use an incremental FM estimator on $\mathbf{S}^N$
presented in \cite{salehian2015efficient}. For completeness, we will
give the formulation of the FM estimator here. Given
$\left\{\mathbf{x}_i\right\}_{i=1}^n$ on $\mathbf{S}^N$, the FM of
these samples can be estimated by $\mathbf{m}_n$, where $\mathbf{m}_n$
is defined recursively as follows:
\begin{align*}
\mathbf{m}_1 &= \mathbf{x}_1 \\
\mathbf{m}_{k+1} &= \Gamma_{\mathbf{m}_k}^{\mathbf{x}_{k+1}}\left(\frac{1}{k+1}\right)
\end{align*}
In \cite{salehian2015efficient}, the authors provide a proof of weak
consistency of this estimator.

Note that in our case, all entries of $\mathcal{F}_i(I)=\sqrt{N_i(I)}$
are positive, so they lie in the positive quadrant of the
hypersphere. Hence the existence and uniqueness of the FM are
guaranteed \cite{afsari2011riemannian}. Given $\left\{\hat{\mu}_{ij},
\hat{\sigma}_{ij}\right\}_{i,j}$, we will learn $\alpha_{i}$ by
minimizing the following objective function,
\begin{align}
\label{theory:loss}
L\left(\left\{\alpha_i\right\}\right) = \frac{1}{|\mathcal{I}^{\text{train}}|}\sum_{k=1}^{|\mathcal{I}^{\text{train}}|} d^2(y_k, p(I_k)).
\end{align}
\vspace*{0.2cm}

{\bf Training of $\left\{\alpha_{i}\right\}$}: $\alpha_{i}$ is the
weight on network $N_i$. Since $\sum_i \alpha_{i} = 1$ and
$\alpha_{i} \geq 0$, we will identify $\left\{\alpha_{i}\right\}$ on
the hypersphere of dimension $m-1$, i.e., on $\mathbf{S}^{m-1}$ and
then do Riemannian gradient descent on the hypersphere. The algorithm
to solve for $\alpha_{i}$ by minimizing $L$ is given in
Algo. \ref{theory:alg1}.

\begin{algorithm}
\KwIn{$\alpha_{i}=1/m$, for all $i$, $\left\{\hat{\mu}_{ij}\right\}$, $\left\{\hat{\sigma}_{ij}\right\}$, $\eta >0$} 
\KwOut{$\left\{\hat{\alpha}_{i}\right\}$}

   $\tilde{\alpha}_{i} = \sqrt{\alpha_{i}}$ and then $\left(\tilde{\alpha}_{i}\right)$ lies on $\mathbf{S}^{m-1}$\;   
    \While{\text{convergence is not achieved}}{ 
   Compute $\nabla_{\left(\tilde{\alpha}_{i}\right)}E \in T_{\left(\tilde{\alpha}_{i}\right)}\mathbf{S}^{m-1}$ \;
   Set $\left(\tilde{\alpha}_{i}\right) = \textsf{Exp}_{\left(\tilde{\alpha}_{i}\right)}\left(-\eta\nabla_{\left(\tilde{\alpha}_{i}\right)}E\right)$ \;
   }
   
   $\hat{\alpha}_{i} = \tilde{\alpha}_{i}^2$, for all $i$ \;
       
    \caption{\footnotesize{Learning of $\left\{\alpha_{i}\right\}$s in order to minimize Eq. \ref{theory:loss}.}}
    \label{theory:alg1}
    
\end{algorithm}
In the above algorithm $\textsf{Exp}$ is Riemannian Exponential map on
hypersphere. This above algorithm ensures that
$\left\{\hat{\alpha}_{i}\right\}$ satisfy the affine constraints.

Since labeled images $(I, y)$ are given, without loss of generality,
we can assume that the label $y$ is of the form
$y=\mathbf{1}_j\in\mathbf{R}^c$ where $I$ is from $j$th class and then
we can view $y$ as a degenerated distribution. To be consistent, we
identify these two distributions, $y$ and $\hat{p}(I)$, with points on
the hypersphere $\mathbf{S}^{c-1}$ and use the arc-length distance as
the distance between $y$ and $\hat{p}(I)$, i.e.
\begin{align*}
d(y,p(I)) = d_{\text{arc}}\left(\frac{y}{\left\Vert y \right\Vert}, \frac{p(I)}{\left\Vert \hat{p}(I) \right\Vert}\right) = \cos^{-1}\left( \frac{y^Tp(I)}{\left\Vert y \right\Vert \left\Vert p(I) \right\Vert}\right)
\end{align*}
\vspace*{0.2cm}

{\bf Prediction of the class for a new sample $I$}: Given
$\left\{\hat{\alpha}_{i}\right\}$, $\left\{\hat{\mu}_{ij}\right\}$,
$\left\{\hat{\sigma}_{ij}\right\}$, the predicted class probability is
given by,
\begin{align*}
\hat{p}_j(I) &= \sum_i \hat{\alpha}_i p_{ij}(\mathcal{F}_i(I)) \\
			  &= \sum_i \hat{\alpha}_i \frac{1}{\hat{C}(\hat{\sigma}_{ij})}\exp\left(-\frac{d_{\text{arc}}^2(\mathcal{F}_i(I),\hat{\mu}_{ij})}{2\hat{\sigma}_{ij}^2}\right) .
\end{align*}
When a test image $I \in \mathcal{I}^{\text{test}}$ is given, we will
assign it to a class $j^*$ for which the prediction probability is
maximized, i.e.,
\begin{align*}
j^* = \arg\max_j \hat{p}_j(I)
\end{align*}

Now, that we have a model and an algorithm to learn the model, we will
present a framework that can combine features extracted from different
algorithms (deep networks or hand-crafted) and hence can have different
number of feature dimension.
\vspace*{0.2cm} 

{\bf $\left\{f_i\right\}$ as the output from the fully connected layer
  (or as hand crafted features):} Note that, $f_i$ is the output of
the network $N_i$ from an intermediate fully connected layer (or $f_i$
be the dimension of hand crafted features). Let, $f_i \in
\mathbf{R}^{d_i}$, for all $i=1, \cdots, m$. We want the features to
be affine invariant, but as none of the networks output affine
invariant features, we quotient out the group of affine
transformations from the features to map each feature on to the
Grassmannian. We want the affine invariance in the extracted features,
so that if two networks (or algorithms to compute hand crafted
features) output features which are related by an affine
transformation, we will not consider these two networks to be
different.

We will use $\mathcal{F}_i$ to denote the point on the Grassmannian
corresponding to $f_i$, i.e., $\mathcal{F}_i \in
\text{Gr}(1,d_i)$. Observe that each $\mathcal{F}_i$ may lie on the
Grassmannian of different dimensions (as $d_i$ may be different for
different networks). Let, $p_{ij}$ be the Gaussian distribution which has been
fitted to $\left\{\mathcal{F}_i\right\}_{I \in \mathcal{I}_j}$
corresponding to $N_i$, i.e., $p_{ij} = \mathcal{N}\left(\mu_{ij},
\sigma_{ij}\right)$, where, $\mu_{ij} \in \text{Gr}(1, d_i)$,
$\sigma_{ij}>0$.

On $\text{Gr}(1,d_i)$, we will use the Gaussian distribution,
$\mathcal{N}\left(\mu_{ij}, \sigma_{ij}\right)$, as defined in
\cite{chakraborty2017statistics}. Let $\mathfrak{x}$ be a $\text{Gr}(1,d_i)$
valued random variable, then the p.d.f. is written as:
\begin{align}
f_X(\mathfrak{x}) = \frac{1}{C(\sigma_{ij})} \exp\left(-\frac{d^2(\mathfrak{x},\mu_{ij})}{2\sigma_{ij}^2}\right),
\end{align}
where, $d$ is the canonical geodesic distance on
$\text{Gr}(1,d_i)$. $C(\sigma_{ij})$ is the normalizing constant. The
canonical distance $d$ on $\text{Gr}(1,d_i)$ is defined as
follows. Let $\mathfrak{x}, \mathfrak{y} \in \text{Gr}(1,d_i)$ with
the respective orthonormal basis $x$ and $y$. Then, the geodesic
distance is defined by:
\begin{align*}
d((\mathfrak{x}, \mathfrak{y}) = \|\arccos\left(\text{diag}\left(\Sigma\right)\right)\|,
\end{align*}
where $U\Sigma V^T = x^Ty$ is the singular value decomposition.

Note that, though, $p_i$ is defined on
$\left\{\mathcal{F}_i\right\}_{I \in \mathcal{I}} \subset
\text{Gr}(1,d_i)$, we will use the support of $p_i$ as $\mathcal{I}$,
i.e.,
\begin{align*}
\label{theory:eq0}
\int_{\mathcal{I}} p_{i} :&= \int_{\mathcal{I}} \frac{1}{k}\sum_j p_{ij} \\
						   :&= \frac{1}{k}\sum_j\int_{\left\{\left.N_i(I)\right\vert I \in  \mathcal{I}_j\right\}} p_{ij} \\
						   :&= \frac{1}{k}\sum_j\int_{\left\{\mathcal{F}_i\right\}_{I \in \mathcal{I}_j}} p_{ij} \\						   &= 1 \numberthis
\end{align*}
The support of $p_{ij}$ over $\mathcal{I}$ is needed to define a
mixture of $\left\{p_{ij}\right\}$ for each $j$.

We define the mixture of $\left\{p_{ij}\right\}$ as $p_j := \sum_i
\alpha_{i} p_{ij}$ for each $j^{th}$ class.

\begin{theorem}
For all $j$, $p_j = \sum_i \alpha_{i} p_{ij}$ is a probability density on $\mathcal{I}_j$.
\end{theorem}
\begin{proof}
For each $j$, 
\begin{align*}
\int_{\mathcal{I}_j} \sum_i \alpha_{i} p_{ij} :&= \sum_i \alpha_{i} \int_{\left\{\left.N_i(I)\right\vert I \in  \mathcal{I}_j\right\}} p_{ij} \\
&= \sum_i \alpha_{i} = 1
\end{align*}
As, $p_{ij} \geq 0$ and $\alpha_{i} \geq 0$, for all $i$, $\sum_i
\alpha_{i} p_{ij} \geq 0$. This completes the proof.
\end{proof}

The above definition of mixture has components defined on different
dimensional spaces, but because of the definition in
Eq. \ref{theory:eq0}, the mixture $p_j = \sum_i \alpha_{ij} p_{ij}$ is
a valid probability density on $\mathcal{I}$ for each $j$. This is a
more general framework as it allows us to combine output of intermediate
layers of deep networks. As future work, we will explore utilizing
this more general framework to combine outputs from intermediate
network layers. As in our experiments, we have found that the choice
of layer for $\left\{f_i\right\}$ is crucial, a detailed study in this
more general direction should be needed and is beyond the scope of
this paper. However, in this work we showed the performance gain of
our proposed framework on hand crafted features such as Histogram of
Oriented gradients (HOG) \cite{dalal2005histograms}, SIFT
\cite{lowe1999object} etc.

%Let, $\left(y_j\left(I\right)\right)$ be the vector consisting the class label for training data $I \in \mathcal{I}^{\text{tr}}$, i.e., $y_j\left(I\right) = 1$ if $I$ belongs to the $j^{th}$ class and $=0$ otherwise. Then, we define the objective function as:
%\begin{align}
%\label{theory:eq2}
%E\left(\mathcal{I}^{\text{tr}}\right) = \sum_{I \in \mathcal{I}^{\text{tr}}} \text{HD}^2\left(\left(y_j\left(I\right)\right), \left(p_j\left(I\right)\right)\right)
%\end{align}
%, where $\text{HD}$ is the Hellinger distance and $p_j(I)$'s are normalized to make $\left(p_j\left(I\right)\right)$ a probability vector.
% 
%, where, 
%\begin{align}
%P\left(I \in \left.\mathcal{I}^{\text{te}}_j \right \vert I\right)  \propto \sum_{i=1}^m \alpha_{ij}  P\left(\left.I\right\vert  I \in \mathcal{I}^{\text{te}}_j, N_i\right) P\left(I \in \mathcal{I}^{\text{te}}_j\right).
%\end{align}
%Here, we know all the probabilities except $P\left(I \in \mathcal{I}^{\text{te}}_j\right)$ which we approximate as $P\left(I \in \mathcal{I}^{\text{tr}}_j\right)$. This approximation assumes that the class distribution for training and testing data partition are same.

\subsection{Non-parametric model}

In the previous subsection, we have assumed a Gaussian distribution on
$\left\{ \mathcal{F}_i(I), I \in \mathcal{I}^{\text{train}}_j\right\}$
for the $i^{th}$ network and $j^{th}$ class. Though this parametric
assumption is simple, it is not very realistic since, the features of
those being classified correctly and those being misclassified are not
from a single Gaussian distribution but maybe a multi-modal
distribution. Hence, in this section, we will estimate
$\left\{p_{ij}\right\}$ using kernel density estimation.  We will
assume Gaussian kernel and write $p_{ij}$ as follows. Let
$\mathcal{F}_{ij} := \left\{\mathcal{F}_i(I)\right\}_{I \in
  \mathcal{I}^{\text{train}}_j}$ be the set of outputs of $N_i$ on
$\mathcal{I}^{\text{train}}_j$.
\begin{align*}
\label{theory:eq3}
p_{ij}(x) = \frac{1}{C(b)\left\vert \mathcal{F}_{ij} \right\vert}\sum_{y \in \mathcal{F}_{ij}}\exp\left(-\frac{d_\text{arc}^2(x,y)}{2b^2}\right)
\end{align*}
for $x \in \left\{ \mathcal{F}_i(I), I \in \mathcal{I}\right\}$.
Here, $b$ is the bandwidth of the kernel which we have selected based
on Silverman's rule of thumb, i.e., $b =
\left(\frac{4\hat{\sigma}_{ij}^5}{3\left\vert \mathcal{F}_{ij}
  \right\vert}\right)^{1/5}$, where, $\hat{\sigma}_{ij}$ is the sample
standard deviation from Eq. \ref{theory: parameter} and
\begin{align*}
\hat{C}(b) = \left[ \sum_{I \in \mathcal{I}^{\text{train}}} \frac{1}{\left\vert \mathcal{F}_{ij} \right\vert}\sum_{y \in \mathcal{F}_{ij}}\exp\left(-\frac{d_\text{arc}^2(\mathcal{F}_i(I),y)}{2b^2}\right)\right]^{-1}.
\end{align*}

The rest of the algorithm is same as in the previous subsection. We
define the mixture of networks model $p_j = \sum_i \alpha_{i} \left(
p_{ij} \circ \mathcal{F}_i \right)$ and then solve for
$\left\{\alpha_{i}\right\}$ in order to minimize the objective
function in Eq. \ref{theory:loss}.

The entire procedure of our ensemble method is shown in Figure \ref{fig:flowchart}.

\begin{figure*}[!ht]
        \centering
                \includegraphics[width=0.9\textwidth]{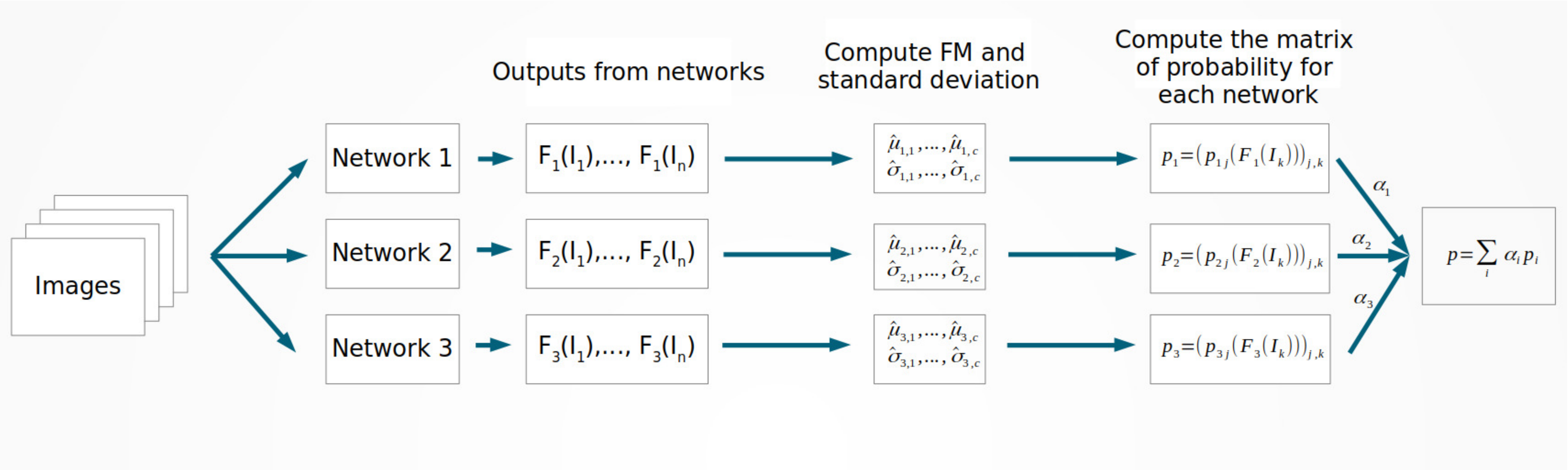}
               \caption{Illustration of our ensemble method.}\label{fig:flowchart}
\vspace*{-0.4em}
\end{figure*} 

\section{Experiments}\label{results}

In this section, we present experiments for both the parametric and
the non-parametric model on four publicly available datasets: CIFAR-10,
CIFAR-100, MNIST, EMNIST-letters (with English alphabet only) \cite{cohen2017emnist},
EMNIST.  A brief description for each of the datasets is given below.
\begin{itemize}
\item The CIFAR-10 dataset consists of 60,000 $32\times 32$ color
  images from 10 classes, of which 50,000 are used for training
  and the rest are used for testing.
\item The CIFAR-100 dataset consists of 60,000 $32\times 32$ color
  images from 100 classes, of which 50,000 are used for training
  and the rest are used as test data.
\item The MNIST dataset consists of 70,000 $28\times 28$ grey images
  of handwritten digits 0~9, of which 60,000 are used for
  training and the rest are used as test data.
\item The EMNIST-letters dataset consists of 145,600 $28\times 28$
  grey images of handwritten English alphabets (26 classes), of which
  124,800 are used for training and the rest for testing.

\item The EMNIST-balanced dataset consists of 131,600 $28\times 28$
  grey images of handwritten alphabets and digits in 47 classes
  (merging those alphabets with similar uppercase and lowercase,
  e.g. C, O), of which 112,800 are used for training and
  the rest for testing.
\end{itemize}
An outline of the entire procedure used in the experiments is presented
below:
\begin{enumerate}
\item Train 20 CNNs $N_{1},...,N_{20}$ for each dataset. The choice of
  CNN can be arbitrary and in order to show the power of our proposed
  ensemble technique, we trained the networks for only a few epochs to
  yield ``weak'' networks.  Here, for the sake of convenience, we
  choose the following architectures (all the models we used in this
  experiment are based on the models provided by keras
  \cite{chollet2015keras} and modified slightly to meet our needs):
\begin{enumerate}
\item \textbf{CIFAR-10} We chose ResNet\cite{he2016deep} with 20
  weight layers and train these networks for only 3 epochs. The
  classification accuracies of these networks range from 61.6\% to
  72.8\% and the average accuracy is 67.02\%.
\item \textbf{CIFAR-100} We chose ResNet with 56 weight layers and
  train these networks for 50 epochs. The classification accuracies of
  these networks range from 59.1\% to 63.5\% and the average accuracy
  is 61.71\%.
\item \textbf{MNIST} We chose a very simple CNN with only one
  convolution layer and one fully-connected layer and train these
  networks for only 1 epoch. The classification accuracies of these
  networks range from 89.8\% to 93.2\% and the average accuracy is
  90.89\%.
\item \textbf{EMNIST-letters} We chose a CNN with 2 convolution layer
  and 2 fully-connected layer and train these networks for only 1
  epoch. The classification accuracies of these networks range from
  89.8\% to 93.2\% and the average accuracy is 90.24\%.
\item \textbf{EMNIST-balanced} We chose a CNN with 2 convolution
  layer and 2 fully-connected layer and train these networks for only
  1 epoch. The classification accuracies of these networks range from
  82.1\% to 83.7\% and the average accuracy is 82.94\%.
\end{enumerate}
\item Compute the estimated weights $\alpha_{i}$, $i=1,...,20$ using
  Algorithm \ref{theory:alg1}.
\item Combine these networks and compute the classification accuracy
  on the test data.
\end{enumerate}
The results are shown in Table \ref{tbl:acc}.

\begin{table}[t]
\begin{center}
\begin{tabular}{c|c|c|c} 
 & Ave. Acc. & Param. & Non-param.\tabularnewline
\hline  
CIFAR-10 & 67.02\% & 75.99\% & 79.5\%\tabularnewline
\hline 
CIFAR-100 & 61.71\% & 65.71\% & 73.14\%\tabularnewline
\hline 
MNIST & 90.89\% & 93.55\% & 93.58\%\tabularnewline
\hline 
EMNIST-letters & 90.24\% & 91.52\% & 91.61\%\tabularnewline 
\hline
EMNIST-balanced & 82.94\% & 84.27\% & 85.66 \%  
\end{tabular}
\end{center}
\caption{Accuracies of the four datasets for parametric and
  non-parametric model}
\label{tbl:acc}
\vspace{-1em}
\end{table}

The result shows clearly that the proposed method works quite well and
as we expected, when the networks are strong there is not much leeway
to improve. On the contrary, when the networks are weak, the
improvement is very significant. We can also see that the difference
between parametric and non-parametric models decreases as the networks
get stronger. Since obviously the features from those being classified
correctly and those being classified incorrectly are not from the same
distribution, in such cases, using a single Gaussian is not
appropriate. When the networks are stronger, the difference between a
single Gaussian distribution and the kernel density estimate is
smaller. The motivation to use the parametric model when it performs
almost as good as non-parametric model is clear: the non-parametric
model takes 2 to 5 times longer than the parametric model.

In practice, we would like to know whether this ensemble technique reduces the time needed to achieve a certain accuracy. To answer this question, we run a experiment based on CIFAR-10 and the parametric ensemble model. The experiment goes as follow:
\begin{enumerate}
\item We trained 5 networks on CIFAR-10 using the same architecture as in the previous experiment.
\item Ensemble the intermediate models after running different number of epochs.  
\end{enumerate}
The result is shown in Figure \ref{fig:acc}. As we can see, the ensemble network performed constantly better. Since our ensemble method requires multiple networks, when comparing the efficiency of our method and the traditional CNN, it is better to consider the effective number of epochs, e.g., if we combine 5 networks and each of them is trained for 10 epochs, then the effective number of epochs would be $5 \times 10 = 50$. Table \ref{tbl:acc2} shows the result of this experiment in terms of the effective number of epochs. The table is to be interpreted as follows: on CIFAR-10, training a network with 50 epochs gives a classification accuracy 76.66\% while training 5 networks, each with 10 epochs, and building the ensemble classifier based on these five networks gives a classification accuracy 80.06\%. The message is that if you train multiple networks and build the ensemble network, you will get a better performance. 

Another advantage of our ensemble method is that we can run
multiple networks on different machines in parallel and then combine
them without any retraining. The extra optimization step for finding the weights
$\left\{ \alpha_i \right\}$ takes less than a few minutes in all our
experiments.

\begin{figure}
  \centering
	\includegraphics[width=3in]{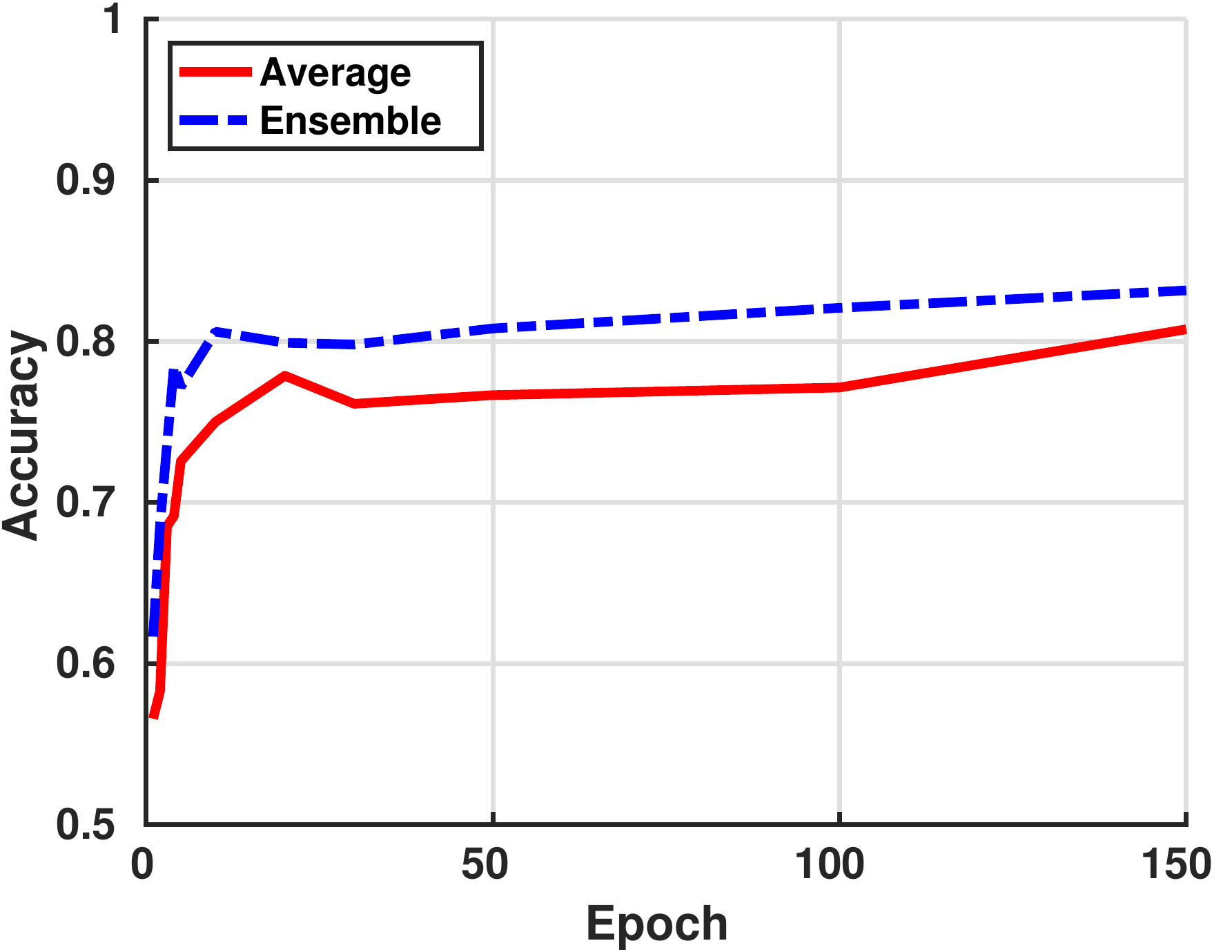} 
  \caption{The comparison between the average accuracy of 5 networks and the accuracy of the ensemble networks: the dashed line is the ensemble network and the solid line is the average accuracy of the 5 networks.}
  \label{fig:acc}
\end{figure}

\begin{table}[t]
\begin{center}
\begin{tabular}{c|c|c|c} 
 Epochs & 20($5 \times 4$) & 50($5 \times 10$) & 100($5 \times 20$) \\% & 150($5 \times 30$) \tabularnewline
\hline  
Ave. & 77.86\% & 76.66\% & 77.13\% \\% & 80.75\%\tabularnewline
\hline 
Ensemble & 78.46\% & 80.6\% & 79.91\% \\ % & 79.8 \%\tabularnewline
\end{tabular}
\end{center}
\caption{Average accuracy of networks and the accuracy of our parametric ensemble network at different effective number of epochs. The ensemble configuration is indicated in the parentheses: number of networks $\times$ number of epochs per network.}
\label{tbl:acc2}
\vspace{-1em}
\end{table}

The third experiment is based ensemble classifiers using the intermediate features instead of the final outputs. The experiment is performed on MNIST, using weak classifiers based on two HOG features \cite{dalal2005histograms} (with two different configuration) and the Daisy feature \cite{tola2010daisy}. Each weak classifier is built using the mixture model described in Section \ref{theory}, i.e., the special case when there is only one network. The average accuracy of these three weak classifiers is 85.16\% and the accuracy of the ensemble classifier is 88.6\%. The result again shows capability of our ensemble method to boost the performance without re-training.

\section{Conclusions}
\label{conc}

In this paper we presented a novel aggregation technique to combine
``weak'' networks/algorithms in order to boost the classification
accuracy over each constituent of the aggregate. Traditional boosting
requires re-training of every constituent of the aggregate and in
contrast, our aggregation model does not require any re-training. This
makes our aggregation model quite attractive from a computational cost
perspective. We presented both parametric and non-parametric
aggregation techniques and demonstrated via experiments the efficiency of the proposed methods. Another key advantage
of our technique stems from the fact that it can cope with aggregation
of features of distinct dimensions that are likely to result from
using either different networks or even hand-crafted features that are
extracted from the data. These salient features make our aggregation
model unique. We presented several experiments demonstrating the
performance of our proposed aggregation technique on widely used
image databases in computer vision literature.
\\

{\bf Acknowledgements:} This research was funded in part by the NSF grant IIS-1525431 and IIS-1724174 to BCV.

{\small
\bibliographystyle{plain}
\bibliography{references}
}

\end{document}